\RequirePackage{fix-cm}
\documentclass[cleveref,12pt,noabbrev,capitalize]{jmlrXiv} %

\usepackage{times}
\usepackage{enumitem}
\usepackage{algorithm}
\usepackage{xfrac}
\usepackage[clean]{macro/main}

\title[Adversarial Learning in Games with Bandit Feedback]{Adversarial Learning in Games with Bandit Feedback: \\ Logarithmic Pure-Strategy Maximin Regret}

\confauthor{%
 \Name{Shinji Ito} \Email{shinji@mist.i.u-tokyo.ac.jp}\\
 \addr The University of Tokyo and RIKEN
 \AND
 \Name{Haipeng Luo} \Email{haipengl@usc.edu}\\
 \addr University of Southern California
 \AND
 \Name{Arnab Maiti} \Email{arnabm2@uw.edu}\\
 \addr University of Washington
 \AND
 \Name{Taira Tsuchiya} \Email{tsuchiya@mist.i.u-tokyo.ac.jp}\\
 \addr The University of Tokyo and RIKEN
 \AND
 \Name{Yue Wu} \Email{wu.yue@usc.edu}\\
 \addr University of Southern California
}

\begin{document}

\maketitle
\ifanonsubmission\else
\blfootnote{Authors are listed in alphabetical order.}
\fi

\begin{abstract}%
Learning to play zero-sum games %
is a fundamental problem in game theory and machine learning.
While significant progress has been made in minimizing external regret in the self-play settings or with full-information feedback,
real-world applications often force learners to play against unknown, arbitrary opponents
and restrict learners to bandit feedback where only the reward of the realized action is observable.
In such challenging settings, it is well-known that $\Omega(\sqrt T)$ external regret is unavoidable (where $T$ is the number of rounds).
To overcome this barrier, we investigate
adversarial learning in zero-sum games under bandit feedback,
aiming to minimize the deficit against the maximin pure strategy --- a metric we term Pure-Strategy Maximin Regret.

We analyze this problem under two bandit feedback models:
\emph{uninformed} (only the realized reward is revealed)
and
\emph{informed} (both the reward and the opponent's action are revealed).
For uninformed bandit learning of normal-form games,
we show that the Tsallis-INF algorithm achieves
$\order(c \log T)$ instance-dependent regret
with a game-dependent parameter $c$.
Crucially, we prove an information-theoretic lower bound showing that
the dependence on $c$ is necessary.
To overcome this hardness,
we turn to the informed setting and introduce \PureUCB{}, which obtains another regret bound of the form $\order(c' \log T)$ for a different game-dependent parameter $c'$ that could potentially be much smaller than $c$. 
Finally, we generalize both results to bilinear games over an arbitrary, large action set,
proposing \TsallisSPM{} and \PureLinUCB{} for the uninformed and informed settings respectively
and establishing similar game-dependent logarithmic regret bounds.
\end{abstract}

\begin{keywords}%
  Learning in Games; Bandit Learning; Instance-Dependent Bound; Zero-Sum Games; Adversarial Learning
\end{keywords}

\section{Introduction}

Learning to play against an opponent is a fundamental problem in game theory and machine learning.
Recent years have seen significant progress on the self-play setting (that is, the same algorithm is deployed by all players) with full gradient feedback,
establishing fast convergence to different equilibrium concepts by showing external or swap regret results that are better than the worst-case $\order(\sqrt{T})$ bounds after $T$ rounds; see e.g.,~\citet{rakhlin2013optimization, syrgkanis2015fast, daskalakis2021near, farina2022near, anagnostides2022near, anagnostides2022uncoupled, soleymani2025cautious}.
More recently, self-play with the more challenging bandit feedback (where only the reward of the selected action is observable) has also been investigated by~\citet{ito2025instancedependent}, demonstrating similar fast convergence that additionally depends on certain game-dependent constants.

However, in many practical situations, self-play becomes an unrealistic requirement.
Instead, very often, the opponents behave in an unknown, arbitrary, or even adversarial way.
In such cases, there is generally no hope to achieve better than $\order(\sqrt{T})$ external/swap regret.
To get around this obstacle, a line of recent works adopted a weaker yet still meaningful performance measure called \textit{Nash-value regret}, which compares the learner's reward to the Nash value of the game.
In particular, 
\citet{maiti2025limitations} showed that game-dependent and a polylogarithmic (in $T$) Nash-value regret bound is achievable, albeit only for $2\times 2$ matrix games with a unique equilibrium.
This raised a natural question: 
{\bfseries Can we achieve similar game-dependent and polylogarithmic regret bounds in general zero-sum games when playing against an arbitrary opponent?}

In this work, we provide an affirmative answer to this question.
Specifically, we focus on a new regret notion that we term \textit{Pure-Strategy Maximin Regret} (PSMR), defined as the accumulated deficit of the learner compared to the pure-strategy maximin value, %
a safety value that a deterministic player can guarantee if the game were known.
Although PSMR is generally weaker than the notion of Nash-value regret, they coincide when the game has a pure-strategy Nash equilibrium (PSNE), a scenario that is already challenging enough and receives significant attention recently~\citep{maiti2023instancedependent, maiti2024nearoptimal, 
ito2025instancedependent}.
Under this metric,
we propose algorithms, analysis, and lower bounds for a comprehensive set of settings, discussed in more detail below.

\paragraph{Contributions and Organization.}
We study two models of bandit feedback: the \emph{uninformed} setting, where the learner observes only noisy rewards of played actions, and the \emph{informed} setting, where the learner additionally observes the opponent's action.
The uninformed feedback model is similar to adversarial bandit and is used in the self-play result of \citet{ito2025instancedependent}, while the informed feedback model is considered by \citet{odonoghue2021matrix} and \citet{maiti2025limitations}.
We present comprehensive results in all four quadrants formed by two feedback models (uninformed versus informed) and two game structures (normal-form versus general bilinear games).
Our main contributions are structured as below.
\begin{itemize}[itemsep=2pt,topsep=2pt]
    \item In \Cref{h1:prelim}, we formalize the problem of adversarial learning in games and define the feedback model and the PSMR.
    \item \Cref{h1:finite-games} focuses on uninformed learning in finite normal-form games, where we show the Tsallis-INF algorithm~\citep{zimmert2021tsallisinf} obtains an $\order(c \log T)$ PSMR bound against any adversary when the game has a strict PSNE, where the constant $c$ depends on the game only (and not the adversary).
    Although similar bounds for Tsallis-INF have been shown for stochastic multi-armed bandits~\citep{zimmert2021tsallisinf} %
    or self-play in games~\citep{ito2025instancedependent},
    our result is not only novel but also surprising since unlike previous work, the environment in which Tsallis-INF operates is completely arbitrary in our setting. %
    Moreover, we prove in \Cref{h2:uninformed-lb} that the dependence on $c$ is unavoidable. 
    More generally, for games without a PSNE, we further prove another game-dependent bound that is independent of $T$. 
    \item To overcome the aforementioned hardness, we turn to the informed learning setting in \Cref{h1:pure-ucb}. We propose an algorithm called \PureUCB{} that plays the pure maximin strategy of the upper confidence bound of the unknown game matrix. %
    We show that \PureUCB{} enjoys a different PSMR bound of the form $\order(c' \log T)$ for another game-dependent constant $c'$ that is incomparable to $c$ in general but could be potentially much smaller. 
    \item Finally, in \Cref{h1:linear-games}, we generalize these results to bilinear games with finite but large action sets. We introduce \TsallisSPM{} and \PureLinUCB{} for the two settings respectively and show that they each obtain similar bounds as their normal-form counterparts, importantly with logarithmic or even no dependence on the number of actions.  %
\end{itemize}

\subsection{Related Work}

\paragraph{Adversarial Bandits.}
The problem of making repeated decisions against a potentially adaptive adversary has been extensively studied in the Multi-Armed Bandit (MAB) framework. Optimal regret bounds of order $\order(\sqrt{T})$ in the worst case are well-established for the adversarial setting, achieved by algorithms like Exp3 and its variants \citep{auer2002nonstochastic,bubeck2012regret}.
A recent focus has been on ``best-of-both-worlds'' algorithms, such as Tsallis-INF by \citet{zimmert2021tsallisinf}, which achieve optimal performance in both stochastic and adversarial bandit settings without knowing the regime. %
Our work builds upon these algorithms, extending them to the game-theoretic setting.

\paragraph{Learning in Games.}
In the context of game theory, significant progress has been made in learning to play games. Starting from \citet{daskalakis2011near,rakhlin2013optimization}, a large body of research analyzes the self-play setting with full gradient feedback, and a series of work improved the regret to a near-optimal $\order(\log T)$ bound; see e.g.,~\citet{syrgkanis2015fast, daskalakis2021near,farina2022near, soleymani2025cautious}. In the special case of two-player zero-sum games, an optimal $\order(1)$ regret bound (focusing on the $T$ dependence only) is achievable~\citep{rakhlin2013optimization, syrgkanis2015fast}. %

Learning in games with bandit feedback is significantly less researched. While standard bandit algorithms can be applied, they often fail to exploit the game-theoretic properties of the setting. When two players collaborate, the sample complexity of identifying a Nash equilibrium with bandit feedback is investigated by \citet{maiti2023instancedependent, maiti2024nearoptimal}. In the self-play setting, a recent work by \citet{ito2025instancedependent} demonstrates that instance-dependent $\order(\log T)$ external regret is achievable. 
An earlier work by~\citet{wei2018more} also studied the self-play setting and obtained $o(\sqrt{T})$ regret, albeit under a stronger feedback model.
When playing against an adversary,
variants of UCB and K-learning were shown to enjoy $\order(\sqrt T)$ worst-case Nash-value regret by \citet{odonoghue2021matrix},
while \citet{maiti2025limitations} proposed an algorithm that achieves instance-dependent $\order(\log T)$ Nash-value regret under the assumption that the game is $2\times 2$ and has a unique Nash equilibrium. We defer more discussion on the feedback models to~\Cref{h1a:uninformed-informed}.

\section{Preliminary}\label{h1:prelim}

\paragraph{Notations.}
We use the notation $\simplex^d$ to denote the $(d-1)$-dimensional standard simplex in $\dR^d$: $\simplex^d = \cbrm[\big]{x\in \dR^d_+ \mid  \sum_{i=1}^d x_i=1}$. We similarly use $\simplex^A$ for a finite set $A$ to denote the set of distributions over $A$. %
A sequence of vectors changing over time is indexed using superscript for the time dimension and subscript for the inherent dimension of the vector; for example, $v!t_i$ is the $i$-th dimension of the vector $v$ at the $t$-th round. %
For an integer $m$, $[m]$ represents the set $\{1, \dots, m\}$.

\paragraph{Problem Statement.}
A learner plays a zero-sum game against an adversary. They can choose their actions from finite action sets $\calX\subset \dR^{d_x}$ and $\calY\subset \dR^{d_y}$ respectively, which have cardinalities $m_x=|\calX|$ and $m_y=|\calY|$.
The game is identified with a utility function $u : \calX\times\calY\to  [-1, 1]$, where $u(x,y)$ is the expected reward received by the learner if she plays $x$ and the adversary plays $y$.
We assume the utility function is bilinear, i.e., $u(x,y)=\inner{x}{Ay}$ for some unknown matrix $A\in \dR^{d_x\times d_y}$.
This general formulation captures all finite normal-form games as special cases where $\calX$ and $\calY$ are standard basis sets.

The learner obtains information about the game by repeatedly playing.
Specifically, in each round $t=1, 2, \dots$, the learner selects an action $x!t \in \calX$, and the adversary simultaneously selects $y!t \in \calY$. We consider an \emph{adaptive adversary} who can select $y!t$ based on the realization of the history, including the learner's past actions
$x!1, \dots, x!{t-1}$.
After both players commit to their actions, the player receives a noisy reward $r_t = u(x!t, y!t) + \eps_t$, assumed to be in $[-1, 1]$. 
The noise sequence $\{\eps_t\}$ is a martingale difference sequence; specifically, conditional on the history and chosen actions, the bias $\EE[\eps_t]$ is $0$. %

\paragraph{Feedback Models.}
We study two feedback settings depending on the information the learner observes after each round:
\begin{itemize}[noitemsep,topsep=3pt]
    \item {\itshape Uninformed}: The learner only observes  the reward $r_t$. This is the same feedback model used in the self-play analysis by \citet{ito2025instancedependent}.
    \item {\itshape Informed}: The learner observes both the reward $r_t$ and the adversary's action $y!t$. This is the feedback model used by \citet{odonoghue2021matrix} and \citet{maiti2025limitations}.
\end{itemize}

\paragraph{Properties of the Game.}
If a pair of actions $(x^*,y^*)\in \calX\times\calY$ satisfies \begin{equation}
    u(x,y^*)\leq u(x^*,y^*)\leq u(x^*,y), \label{eq:ne}
\end{equation}
for any $x\in \calX$ and $y\in \calY$, we say that $(x^*,y^*)$ is a \emph{pure-strategy Nash equilibrium} (PSNE). If \eqref{eq:ne} holds with strict inequality whenever $x\neq x^*$ and $y\neq y^*$, we say the pair is a \emph{strict} PSNE.

For a pair of distributions over actions $(p, q)\in \simplex^\calX\times \simplex^\calY$, we slightly abuse the notation of the utility function and define $u(p,q)=\EE_{x\sim p,y\sim q}[u(x,y)]=\sum_{x\in \calX}\sum_{y\in \calY}p_x q_y u(x,y)$ to be the expected reward of the learner if she plays a sample from $p$ and the adversary plays a sample from $q$. If a pair of distributions over actions $(p^*, q^*)\in \simplex^\calX\times \simplex^\calY$ satisfies
\begin{equation}
    u(p,q^*)\leq u(p^*,q^*)\leq u(p^*,q),
\end{equation}
for any $p\in \simplex^\calX$ and $q\in \simplex^\calY$, we say that $(p^*,q^*)$ is a \emph{mixed-strategy Nash equilibrium} (MSNE). The celebrated minimax theorem by \citet{vonneumann1928theory} proves that an MSNE always exists, and that all Nash equilibria yield the same expected value $\vMix=u(p^*,q^*)$, which is also equal to $\max_{p\in \simplex^\calX}\min_{q\in \simplex^\calY} u(p, q)$ and $\min_{q\in \simplex^\calY}\max_{p\in \simplex^\calX} u(p, q)$.

In contrast, central to our work is the concept of pure-strategy maximin value, defined as $\vStar\defeq \max_{x\in \calX} \min_{y\in \calY} u(x,y)$, which is the maximum utility a deterministic learner can guarantee in the worst case.
By definition, $\vStar$ is no more than $\vMix$, but when a PSNE exists, these two values coincide.

\paragraph{Goal.}
From the learner's perspective, this problem can be seen as an instance of adversarial multi-armed bandits~\citep{auer2002nonstochastic}, where the standard goal is to minimize \textit{external regret}, which compares the learner's total reward to that of the best fixed action in hindsight.
Specifically, let $\ER_T(x) \defeq \EE\sbrm[\Big]{\sum_{t=1}^T \rbrm[\big]{u(x,y!t)-u(x!t,y!t)}}$ be the regret compared to action $x$.
Then external regret is defined as
$\ER_T \defeq \max_{x\in \calX} \cbrm[\big]{\ER_T(x)}$.
Applying standard algorithms such as Exp3~\citep{auer2002nonstochastic} guarantees $\ER_T = \tilde{\order}(\sqrt{m_xT})$, which is generally not improvable even for simple games, as shown in~\citet[Theorem~3]{maiti2025limitations}.

To get around this obstacle,
we propose to consider the following new regret metric, termed \emph{Pure-Strategy Maximin Regret} (PSMR), as the learner's performance measure:
\begin{equation}
    \psmr_T\defeq \EE\sbrm[\Bigg]{\sum_{t=1}^T \rbrm[\big]{\vStar-u(x!t,y!t)}},
    \label{eq:psmr-def}
\end{equation}
where the expectation is taken over the randomness of the algorithm, the noise in the reward, and potentially the adversary.
In words, PSMR measures the learner's accumulated deficit compared to the pure-strategy maximin value.
On the other hand, the Nash-value regret considered in prior work such as~\citet{odonoghue2021matrix, maiti2025limitations} is defined as $\NR_T \defeq\EE\sbrm[\Big]{\sum_{t=1}^T \rbrm[\big]{\vMix-u(x!t,y!t)}}$, which measures the learner's accumulated deficit compared to the Nash value (equivalently, the mixed-strategy maximin value).

By definition, it is clear that $\psmr_T \leq \NR_T \leq \ER_T$; the first inequality is due to $\vStar\leq \vMix$, and the second due to $\NR_T=\EE\sbrm[\Big]{\sum_{t=1}^T \rbrm[\big]{u(p^*, q^*)-u(x!t,y!t)}}\leq\EE_{x\sim p^*}[\ER_T(x)]\leq \ER_T$ for any NE $(p^*,q^*)$.
When the game has a PSNE, however, $\psmr_T$ and $\NR_T$ coincide. 

Finally, we point out that in the special case when the opponent has only one action ($|\calY|=1$), our problem is simply the standard stochastic linear bandit~\citep{abbasi-yadkori2011improved} or
stochastic multi-armed bandit~\citep{lai1985asymptotically} if $\calX$ is the standard basis.
In this case,
both $\psmr_T$ and $\NR_T$ recover $\ER_T$.
Therefore, our results can also be seen as a generalization of the classical instance-dependent bounds for stochastic bandits achieved by, for example, the UCB algorithm~\citep{auer2002using}.

\section{Uninformed Learning in Normal-Form Games}\label{h1:finite-games}

In this section, we focus on the uninformed setting and present our results for normal-form games, where $\calX$ and $\calY$ are standard basis sets in $\dR^{m_x}$ and $\dR^{m_y}$. In this case, $A$ is the standard utility matrix, with utilities given by $u(e_i,e_j)=A_{ij}$.
Mixed strategies correspond to vectors $p,q$ in the simplexes $\simplex^{m_x}, \simplex^{m_y}$, with expected utility $u(p,q)=\inner{p}{Aq}$.

\subsection{Results and Analysis for Tsallis-INF}\label{h2:tsallis-inf}
We consider using the Tsallis-INF algorithm of~\citet{zimmert2021tsallisinf} in this setting. 
Tsallis-INF was originally proposed for the multi-armed bandit problem and shown to be optimal for both the stochastic setting and the adversarial setting.
Since then, it has been extended and analyzed in many other settings, including self-play in games~\citep{ito2025instancedependent}.
However, in all previous studies, achieving logarithmic regret bounds requires either a stationary or a self-play environment, 
and it is unclear at all whether it still works in our problem where the environment is controlled by an arbitrary opponent.
Somewhat surprisingly, we show that it does still achieve a certain game-dependent logarithmic regret in our problem.

Specifically, Tsallis-INF is an instance of the well-known Follow-the-Regularized-Leader (FTRL) framework.
It uses the negative Tsallis entropy as the regularizer, defined as (for a parameter $\alpha \in (0,1)$):
\begin{equation}
\varphi_{\alpha}(p)=\frac{1}{\alpha} \sum_{i} \rbr{p_i^\alpha - p_i},\label{eq:tsallis}
\end{equation}
along with standard importance-weighted reward estimators and a simple time-varying learning rate $\eta_t$.
See \Cref{alg:tsallis-inf} for the complete pseudocode.

\begin{algorithm}[t]
\caption{Tsallis-INF}\label{alg:tsallis-inf}
\begin{algorithmic}[1]
\Require A Tsallis entropy parameter $\alpha\in [0, 1]$ and a sequence of learning rates $\{\eta_t\}$.
\For{$t=1,2,\dots$}
  \State Let $p!t = \argmax_{p\in \simplex^{m_x}} \cbrm[\big]{\inner{p}{\sum_{s<t} g!s} + \frac{1}{\eta_t} \varphi_\alpha(p)}.$
  \State Sample $x!t\in \calX$ from the distribution specified by $p!t$ and play it.
  \State Observe the reward $r_t$.
  \State Set the reward estimator $g!t\in \dR^{\calX}$, where
      $g!t_x = \begin{cases}
          1-(1-r_t) / p!t_x, & \text{if } x!t=x, \\
          1, & \text{otherwise.}
      \end{cases}$
\EndFor
\end{algorithmic}
\end{algorithm}

We show that Tsallis-INF achieves different game-dependent logarithmic PSMR bounds for the case when: a) the game has a strict PSNE, or b) no PSNE exists. If a strict PSNE $(x^*, y^*)$ exists, the bound depends on the following sub-optimality gap parameters:
\begin{align*}
    \DeltaR_x &= u(x^*,y^*)-u(x,y^*), \quad
    \DeltaC_y = u(x^*,y)-u(x^*,y^*).
\end{align*}
From the definition of strict PSNEs, we know that these parameters are strictly positive. We further define $\DRmin = \min_{x\neq x^*}\DeltaR_x$ and $\DCmin = \min_{y\neq y^*}\DeltaC_y$. 
We point out that these parameters also characterize the sample complexity of finding a PSNE~\citep{maiti2024nearoptimal} and the external regret for self-play~\citep{ito2025instancedependent}.

On the other hand, if PSNEs do not exist, we know that the Nash value $\vMix$ is strictly higher than the pure-strategy maximin value $\vStar$, and our bound depends on the gap between them: $\DeltaMix\defeq\vMix-\vStar>0$.
More concretely, our main result is summarized in the following theorem.

\begin{theorem}\label{thm:tsallis} 
    In every game $u(\cdot, \cdot)$, against any adaptive adversary, Tsallis-INF with $\alpha=\frac 12$ and $\eta_t = \frac{1}{2\sqrt{t}}$ achieves a regret bound of $\psmr_T = \order\rbrm[\big]{\sqrt{m_x T}}$. Additionally,

    \begin{itemize}[topsep=0pt,noitemsep]
        \item If the game has a strict PSNE, we further have $\psmr_T = \order\rbrm[\big]{\frac{1}{\DCmin} \sum_{x\neq x^*} \frac{\log T}{\DeltaR_x}}$.
        \item If the game has no PSNEs, we further have $\psmr_T = \order\rbrm[\big]{\frac{m_x }{\DeltaMix}}$.
    \end{itemize}
\end{theorem}

\begin{proof}
Since our model is a special case of adversarial bandit, Theorem~1 of \citet{zimmert2021tsallisinf} directly implies that $\psmr_T\leq \ER_T\leq 8\sqrt{m_x T}+2$ (the difference in constant factors is due to our reward range $[-1,1]$ being twice as wide as theirs).

For the instance-dependent PSMR in the no-PSNE case, we observe that $\NR_T-\psmr_T=\rbrm[\big]{\vMix-\vStar}T=\DeltaMix T$. Together with the $\ER_T$ bound, we have
\begin{align*}
    \psmr_T=\NR_T-\DeltaMix T\leq \ER_T-\DeltaMix T=2+8\sqrt{m_x T}-\DeltaMix T.
\end{align*}
The function $T\mapsto 8\sqrt{m_x T}-\DeltaMix T$ is maximized at $T=\frac{16 m_x}{(\DeltaMix)^2}$ with a value of $\frac{16 m_x}{\DeltaMix}$ (see \Cref{lem:sqrt-func}), so $\psmr_T\leq \frac{16 m_x}{\DeltaMix}+2$.

For the instance-dependent PSMR in the strict PSNE case, let $(x^*, y^*)$ be the strict PSNE and $C\in [0, T]$ be the expected number of times the adversary deviates from the NE, that is, 
$
C\defeq\EE\sbrm[\Big]{\sum_{t=1}^T \one[y!t \neq y^*]}.
$
We then have
\begin{align*}
    \ER_T(x^*)-\psmr_T
    & = \EE\sbrm[\bigg]{\sum_{t=1}^T \rbrm[\Big]{u(x^*,y!t)-\vStar}}
      = \EE\sbrm[\bigg]{\sum_{t=1}^T \rbrm[\Big]{u(x^*,y!t)-u(x^*,y^*)}}\\
    & = \EE\sbrm[\bigg]{\sum_{t=1}^T \one[y!t \neq y^*] \DeltaC_{y!t}}
      \geq C \DCmin. \yestag\label{eq:er-nr-difference}
\end{align*}

Further,
\begin{align*}
    \EE\bigg[
        \sum_{t=1}^T \anglem[\big]{\DeltaR,p!t}
        \bigg] - \ER_T(x^*)
    & = \EE\sbrm[\bigg]{\sum_{t=1}^T \DeltaR_{x!t}} 
      - \EE\sbrm[\bigg]{\sum_{t=1}^T \rbrm[\Big]{u(x^*,y!t)-u(x!t,y!t)}}
      \\
    & = \EE\bigg[
        \sum_{t=1}^T \rbrm[\Big]{{
        u(x^*,y^*)-u(x^*,y!t)+u(x!t,y!t)-u(x!t,y^*)}}
        \bigg].
\end{align*}
Each summand in the RHS is equal to $0$ when $y!t=y^*$, and bounded by $4$ otherwise, so the entire expectation is bounded as $\EE\sbrm[\big]{
        \sum_{t=1}^T \anglem[\big]{\DeltaR,p!t}
        } - \ER_T(x^*) \leq 4C$.
Next, we apply the following key external regret bound
of Tsallis-INF (see e.g.,~\citealp[Theorem 1]{ito2025instancedependent}):
\begin{align*}
    \ER_T(x^*) 
    \leq \ER_T 
    & \leq 19 \EE\sbrm[\Bigg]{
        \sum_{t=1}^T \frac{1}{\sqrt{t}}
        \sum_{x\neq x^*} \sqrt{p!t_x}
        }.
\end{align*}
This is further bounded by Cauchy-Schwarz as
\begin{align*}
    \ER_T(x^*) & \leq 19 \EE\sbrm[\Bigg]{
        \sum_{t=1}^T
        \sum_{x\neq x^*} \sqrt{\frac{1}{t\DeltaR_x} \DeltaR_x p!t_x}
    } 
    \leq 19 \EE\sbrm[\Bigg]{\sqrt{
        \rbrm[\Bigg]{\sum_{t=1}^T
        \sum_{x\neq x^*} \frac{1}{t\DeltaR_x}} \rbrm[\Bigg]{\sum_{t=1}^T
        \sum_{x\neq x^*} \DeltaR_x p!t_x}
        }
    } \\
    & \leq 19 \EE\sbrm[\Bigg]{\sqrt{
        \sum_{x\neq x^*} \frac{H_T}{\DeltaR_x}
        \sum_{t=1}^T \anglem[\big]{\DeltaR, p!t}
        }
    }
    \leq 19 \EE\sbrm[\Bigg]{\sqrt{
        \sum_{x\neq x^*} \frac{H_T}{\DeltaR_x}
        \rbrm[\big]{\ER_T(x^*)+4C}
        }
    } \\
    & \leq 361 \sum_{x\neq x^*} \frac{H_T}{\DeltaR_x} + 38 \sqrt{C \sum_{x\neq x^*} \frac{H_T}{\DeltaR_x}},
\end{align*}
where $H_T=\sum_{t=1}^T \frac{1}{t} \leq 1+ \log T$ is the sum of a harmonic series, and the last step is a standard self-bounding argument and can be shown with
elementary algebra (\Cref{lem:self-bound}).
Combining this with \eqref{eq:er-nr-difference}, we can bound $\psmr_T$ as follows:
\begin{align*}
    \psmr_T \leq \ER_T(x^*) - C \DCmin \leq 361 \sum_{x\neq x^*} \frac{H_T}{\DeltaR_x} + 38 \sqrt{C \sum_{x\neq x^*} \frac{H_T}{\DeltaR_x}} - C \DCmin.
\end{align*}
Another application of \Cref{lem:sqrt-func} finally eliminates $C$ and proves our desired bound in the strict PSNE case,\footnote{Throughout this paper, no attempts are made to optimize global constants.}
\let\jmlrQED\relax %
\begin{align*}
     \psmr_T \leq 361 \rbrm[\Big]{1+ \frac{1}{\DCmin}} \sum_{x\neq x^*} \frac{H_T}{\DeltaR_x} \leq 361 \rbrm[\Big]{1+ \frac{1}{\DCmin}} \sum_{x\neq x^*} \frac{1 + \log T}{\DeltaR_x}.\mqed
\end{align*}
\end{proof}

From the proof of \Cref{thm:tsallis}, it is clear that the $\order\rbrm[\big]{\frac{m_x}{\DeltaMix}}$ instance-dependent bound in fact holds for any algorithms with a worst-case $\order(\sqrt{m_x T})$ external regret bound.
Also, for the case when a non-strict PSNE exists, $\psmr_T=\Omega(\sqrt{T})$ is in fact unavoidable; see \Cref{rem:sqrt-t-lb} below.

There is no prior work we can directly compare our bounds to for this uninformed setting.
However, since our results obviously also apply to the easier informed setting, which~\citet{maiti2025limitations} studied, we make the following comparisons to their bound.

As mentioned, \citet{maiti2025limitations} only considered a $2\times 2$ zero-sum game $A=\rbrm[\Big]{\begin{matrix}a & b \\ c & d\end{matrix}}$ with a unique NE,
and they achieved $\NR_T = \order\rbrm[\big]{\frac{\log T}{\DeltaEntry^3}+\frac{\log^2 T}{\DeltaEntry^2}}$ for another game-dependent constant $\DeltaEntry=\min\{|a-b|,|a-c|,|b-d|,|c-d|\}$.
Besides the fact that our results hold without the uniqueness assumption and that a worst-case guarantee  $\order(\sqrt{T})$ always holds, our bounds are also better in the following sense:
\begin{itemize}[itemsep=2pt,topsep=2pt]
    \item When the unique NE is pure, $\psmr_T$ is equal to $\NR_T$, and since $\DCmin\geq \DeltaEntry$ and $\DeltaR_x\geq \DeltaEntry$, our bound $\frac{1}{\DCmin} \sum_{x\neq x^*} \frac{\log T}{\DeltaR_x}$ is at most $\frac{\log T}{\DeltaEntry^2}$, strictly improving upon theirs.
    \item When the unique NE is mixed, our $\psmr_T$ metric is weaker than their $\NR_T$, but at the same time, our bound is also strictly better since $\frac{m_x}{\DeltaMix}\leq\frac{8}{\DeltaEntry^2}$ by the fact $\DeltaMix \geq \frac{\DeltaEntry^2}{4}$ from \Cref{lem:delta-comparison}.
\end{itemize}

\subsection{Lower Bound}\label{h2:uninformed-lb}

In the case when a strict PSNE exists, \citet{ito2025instancedependent} proved that Tsallis-INF enjoys $\ER_T=\order\rbrm[\big]{(\sum_x \frac{1}{\DeltaR_x}+\sum_y \frac{1}{\DeltaC_y})\log T}$ in the self-play setting.
While not directly comparable to our setting, it naturally makes one wonder whether our bound $\order\rbrm[\big]{\frac{1}{\DCmin} \sum_{x\neq x^*} \frac{\log T}{\DeltaR_x}}$ is improvable.
Indeed, their dependency on the gap parameters is of order $\frac{1}{\DRmin}+\frac{1}{\DCmin}$, thus strictly better than ours, which is of order $\frac{1}{\DRmin\DCmin}$.
It turns out that such $\frac{1}{\DRmin\DCmin}$ dependency is unavoidable in our setting, meaning that uninformed adversarial learning in games is inherently more difficult than self-play. This is summarized in the following theorem.

\begin{theorem}\label{thm:uninformed-lb}
    For small enough $\DRmin, \DCmin$, large enough $T$, and any uninformed learner, there exist an adversary and a game $u(\cdot, \cdot)$ with $m_x=m_y=2$ and a strict PSNE, such that the learner's PSMR is at least $\Omega\rbrm[\big]{\min\cbrm[\big]{\frac{1}{\DRmin\DCmin},\sqrt{T}}}$.
\end{theorem}

In the following, we will show a simplified, informal argument for the $\Omega(\frac{1}{\DRmin\DCmin})$ lower bound, using big-$\order$ to hide constant factors. A full proof with a careful analysis of constant factors is deferred to \Cref{h1a:proof-lower-bound}. \\

\begin{proof}(sketch).
    In this sketch, we focus on the more interesting case when $\frac{1}{\DRmin\DCmin}=\order(\sqrt T)$ and show that $\psmr_T = \Omega\rbrm[\big]{\frac{1}{\DRmin\DCmin}}$.
    Consider a game matrix
    $
        A=\rbrm[\bigg]{\begin{matrix}
            0 &        \DeltaC \\ 
            -\DeltaR & K - \DeltaR 
          \end{matrix}},
    $
    parameterized by $K$, $\DeltaR$, and $\DeltaC$. We assume that all entries are in $[-1, +1]$ and $\DeltaR, \DeltaC>0$. We know that $A$ admits a strict PSNE $(x^*,y^*)=(e_1,e_1)$ with a value of $v^*=0$. The two gap variables $\DeltaR,\DeltaC$ are exactly $\DRmin$ and $\DCmin$ of $A$. We set the noise model to be binary on $\{\pm 1\}$, i.e., at round $t$, the reward comes from a two-point distribution so that $r_t\in \{-1,+1\}$ and $\EE[r_t]=u(x!t,y!t)$.
    
    The adversary is parameterized by $\eps\in (0,1)$ and $T'\in \bbN_+$, $T'\leq T$, where he plays a sample from the distribution $y!t\sim q!t=(1-\eps,\eps)$ for $t\leq T'$ and $y!t=y^*$ for $t>T'$. Then, for $t\leq T'$, we have
    $
        A q!t = \rbrm[\bigg]{\begin{matrix}
            \eps \DeltaC\\
            K \eps - \DeltaR
        \end{matrix}}.
    $

    Let $\delta=\eps\DeltaC-K \eps+\DeltaR$. In each round up to $t=T'$, the KL divergence between the reward distributions of the two actions of the learner can be bounded by $\order(\delta^2)$ (see \Cref{lem:kl-bernoulli}). Using standard arguments (see Chapter 15.1 of \citet{lattimore2020bandit}), the total KL divergence between the two arms at the end of the $T'$-th round is bounded by $\delta^2T'$. 

    We let $K=1$ and $\eps=\DeltaR-12\DeltaR\DeltaC$, and thus $\delta=\DeltaR\DeltaC(13-12\DeltaC)=\Theta(\DeltaR\DeltaC)$. By choosing some $T'\leq T$ such that $T'=\Theta\rbrm[\big]{\frac{1}{(\DeltaR\DeltaC)^2}}$ and $\delta^2 T'\leq 1$, which is possible because $\frac{1}{\DeltaR\DeltaC}=\order(\sqrt T)$, we can see that the total KL divergence is $\order(1)$.
    This means that the learner cannot distinguish between the two actions beyond a constant error probability, and that the worse action is selected at least $\frac1{12}$ of the time. In this case, the accumulated PSMR up to time $T'$ is at least
    \begin{align*}
        \psmr_{T'}&\geq T'\rbrm[\Big]{\frac{1}{12} (\DeltaR-K\eps)-\frac{11}{12} \eps\DeltaC} \yestag\label{eq:psmr-lb-main}\\
        &\geq \frac{T'}{12}\DeltaR\DeltaC=\Theta\rbrm[\big]{\frac{1}{\DeltaR\DeltaC}}.
    \end{align*}

    Finally, note that for $t\in [T'+1,T]$, the adversary plays the PSNE action and thus the PSMR can only increase, proving that $\psmr_T$ is indeed also of order $\frac{1}{\DRmin\DCmin}$.
\end{proof}

\begin{remark}\label{rem:sqrt-t-lb}
    In the construction above, we can let $\DeltaR=\DeltaC=0$, $K=-1$, and let the adversary play $(1-\eps, \eps)$ with $\eps=\frac{1}{\sqrt T}$ for all $T'=T$ rounds. We can verify that the accumulated KL divergence is $\order(1)$ as well. 
    This matrix, $\rbrm[\Big]{\,\begin{matrix}
        0 & 0 \\ 
        0 & -1
        \end{matrix}}$,
    has a non-strict PSNE, meaning that $\DRmin=\DCmin=\DeltaMix=0$, and \eqref{eq:psmr-lb-main} shows that $\psmr_T = \Omega(\sqrt{T})$.
    This shows that, beyond the two cases considered in \Cref{thm:tsallis} where we prove game-dependent logarithmic PSMR, any algorithm must suffer $\psmr_T=\Omega(\sqrt T)$.
\end{remark}

\section{Informed Learning in Normal-Form Games}\label{h1:pure-ucb}

The previous section demonstrates a gap between adversarial and self-play learning in games.
To improve performance in the adversarial bandit setting, we turn to the informed setting and grant the algorithm access to extra side information, the opponent's action $y!t$. 
Note that previous work~\citep{odonoghue2021matrix,maiti2025limitations} has implicitly assumed this feedback model.

\begin{algorithm}[t]
\caption{\PureUCB}\label{alg:pure-ucb}
\begin{algorithmic}[1]
\Require A high-probability parameter $\delta$.
\State Initialize the accumulated rewards $R!0(x,y)=0$, and the number of pulls $N!0(x,y)=0$.
\For{$t=1,2,\dots$}
  \State Define the UCB for any pair $(x,y)$ as $U!{t}(x, y)\gets 1$ if $N!{t-1}(x,y)=0$, or else\\
  {$U!{t}(x, y)\gets
    \frac{R!{t-1}(x,y)}{N!{t-1}(x,y)}+\sqrt{\frac{4\log (1/\delta)+2\log (1+N!{t-1}(x,y))}{N!{t-1}(x,y)}}
  $} if $N!{t-1}(x,y)>0$. \label{code:ucb-confidence}
  \State Let $x!t=\argmax_{x\in \calX}\min_{y\in \calY} U!{t}(x,y)$ with ties broken arbitrarily.\label{code:ucb-maximin}
  \State Play $x!t$. Observe the adversarial action $y!t$ and the reward $r_t$.
  \State Update $N!t(x!t,y!t)\gets N!{t-1}(x!t,y!t)+1$. \label{code:ucb-n}
  \State Update $R!t(x!t,y!t)\gets R!{t-1}(x!t,y!t)+r_t$. \label{code:ucb-sum}
  \State All other entries of $N!t$ and $R!t$ are the same as $N!{t-1}$ and $R!{t-1}$.
\EndFor
\end{algorithmic}
\end{algorithm}

With such extra information, it becomes natural to directly build estimates on the game matrix $A$ and apply the optimism principle as in the standard UCB algorithm~\citep{auer2002using}.
This leads to our proposed algorithm \PureUCB{} (\Cref{alg:pure-ucb}). 
Specifically, at the beginning of each round $t$, it constructs an estimate $U!t(x,y)$ for every action pair $(x,y)$, representing an upper confidence bound of the utility $u(x,y)$. %
The algorithm then simply plays the optimistic pure maximin strategy $\argmax_{x\in \calX}\min_{y\in \calY} U!{t}(x,y)$ in round $t$, that is, the action that maximizes the worst-case optimistic reward.

We point out that  \PureUCB{} is very close to Algorithm~1 of \citet{odonoghue2021matrix},
with the only difference being that their algorithm plays the optimistic \textit{mixed} maximin strategy $\argmax_{p\in \simplex^{m_x}}\min_{y\in \calY} U!{t}(p,y)$ instead.
While they only prove that their algorithm achieves
$\NR=\tilde{\order}(\sqrt{m_x m_y T})$ (which is in fact even worse than using a standard adversarial bandit algorithm), 
we prove the following game-dependent logarithmic PSMR for our slightly different algorithm \PureUCB{}, using yet another sub-optimality gap parameter: $\Delta_{xy}=\vStar-u(x,y)$ for all $(x,y)\in \calX\times \calY$.
(Note that $\Delta_{xy}$ can be negative.)

\begin{theorem}\label{thm:pureucb}
    For any game $u$, \PureUCB{} with $\delta=\frac{1}{T}$ simultaneously achieves an instance-dependent regret of $\psmr_T=\order\rbrm[\big]{\sum_{x, y : \Delta_{xy} > 0} \rbrm[\big]{\Delta_{xy}+\frac{1}{\Delta_{xy}}\log T}}$ and a worst-case bound of $\psmr_T=\order\rbrm[\big]{\sqrt{m_x m_y T \log T}}$, against any adaptive adversary.
\end{theorem}

The proof mostly extends standard UCB proofs and argues that each $(x,y)$ pair can be encountered at most  $\tilde{\order}\rbrm{1/\Delta_{xy}^2}$ if its gap $\Delta_{xy}>0$; see \Cref{h1a:proof-pureucb} for a rigorous proof.
Note that unlike \Cref{thm:tsallis} where different cases were discussed, the results of \Cref{thm:pureucb} hold always (regardless whether a PSNE exists or not).
The instance-dependent bounds in these two theorems are generally incomparable (setting aside the fact that their feedback models are also different):
\begin{itemize}[itemsep=2pt,topsep=2pt]
    \item There are games where Tsallis-INF achieves a much better bound. For example, consider $A=\rbrm[\Big]{\begin{matrix}
        0 & 1 \\ 
        -1 & -\eps
        \end{matrix}}$ for a small $\eps > 0$,
        in which case Tsallis-INF achieves $\psmr_T=\order(\log T)$ since $\DRmin=\DCmin=1$, 
        while \PureUCB{} achieves $\order(\frac{\log T}{\eps})$ since $\Delta_{2,2} = \eps$.

     \item Conversely, there are also games where \PureUCB{} achieves a much better bound. For example, consider $A=\rbrm[\Big]{\begin{matrix}
    0 & \eps \\ 
    -1 & 1
    \end{matrix}\,}$ for a small $\eps > 0$,
    in which case Tsallis-INF achieves $\psmr_T=\order(\frac{\log T}{\eps})$ since $\DRmin=\eps$ and $\DCmin=1$, 
    while \PureUCB{} achieves $\order(\log T)$ since the only non-negative gap is $\Delta_{2,1} = 1$.
\end{itemize}

Similarly, using the same examples, we see that \Cref{thm:pureucb} and the result of~\citet{maiti2025limitations} are also directly incomparable:
in the first example, we have $\DeltaEntry=1-\eps$ and thus \citet{maiti2025limitations} achieve $\psmr_T=\NR_T=\order(\log^3 T)$, better than \PureUCB{}'s bound $\order(\frac{\log T}{\eps})$;
in the second example, we have $\DeltaEntry=\eps$ and thus \citet{maiti2025limitations} achieve $\psmr_T=\NR_T=\order\rbrm[\big]{\frac{\log T}{\eps^3}+\frac{\log^2 T}{\eps^2}}$, much worse than \PureUCB's bound $\order(\log T)$. %

\section{Extension to General Bilinear Games}\label{h1:linear-games}

In this section, we return to the general bilinear setting, where the action sets $\calX \subset \dR^{d_x}, \calY \subset \dR^{d_y}$ are arbitrary finite sets with large cardinalities $m_x$ and $m_y$.
We assume without loss of generality that $\calX$ and $\calY$ span $\dR^{d_x}$ and $\dR^{d_y}$, respectively. We further assume that $\normm{x}_2\leq 1$, $\normm{y}_2\leq 1$ for any $x\in \calX$ and $y\in \calY$, and $\normm{A}_2 \leq 1$; this implies $u(x,y)=\inner{x}{Ay}\in [-1, 1]$.
This framework generalizes the stochastic linear bandit problem \citep{abbasi-yadkori2011improved}, which is recovered when $\calY$ is a singleton ($|\calY|=1$).
Similar to the motivation behind linear bandits, we investigate this setting to derive
regret bounds with dominant terms that scale polynomially with dimensions $d_x, d_y$, but only logarithmically with the action set sizes $m_x,m_y$.

\subsection{Uninformed Learning in Bilinear Games with \TsallisSPM{}}\label{h2:tsallis-spm}

As Tsallis-INF is shown to be successful in normal-form games (\Cref{h2:tsallis-inf}), it is more than natural to extend it to bilinear games.
Indeed, for adversarial linear bandits, \citet{ito2024adaptive} have generalized Tsallis-INF and proposed an algorithm, which we will refer to as \TsallisSPM, to achieve similar best-of-both-worlds results.
We consider using the same algorithm for uninformed learning in general bilinear games, shown in \Cref{alg:tsallis-spm}.
It is again based on the FTRL framework,
but uses a variant of the Tsallis entropy regularizer (\Cref{line:FTRL}),
a reward estimator tailored to the linear structure (\Cref{line:estimator}), and a more sophisticated exploration and learning rate schedule (\Cref{line:explore} and \Cref{line:learning_rate}).

\begin{algorithm}[t]
\caption{\TsallisSPMLong}\label{alg:tsallis-spm}
\begin{algorithmic}[1]
\Require A Tsallis entropy parameter $\alpha\in (0, 1)$ (\textit{c.f.} \eqref{eq:tsallis}), learning rate parameters $\beta_1>0$, $\bar{\beta}\geq 0$, a fixed exploration distribution $\pzero$ with a variance ratio $c>0$.
\For{$t=1,2,\dots$}
  \State Let $\phat!t = \argmax_{p\hspace{0.08em}\in \simplex^{m_x}} \cbrm[\big]{\inner{p}{\sum_{s<t} g!s} + \beta_t \varphi_\alpha(p)+\bar{\beta} \varphi_{1-\alpha}(p)}.$
  \State Set $z_t=\frac{d_x\rbr{\min\{\normm{\phat!t}_\infty,1-\normm{\phat!t}_\infty\}}^{1-\alpha}}{1-\alpha}$.\label{line:FTRL}
  \State Set the exploration coefficient $\gamma_t=\frac{4c z_t}{\beta_t}$.
  \State Compute the distribution $p!t = (1-\gamma_t) \phat!t + \gamma_t \pzero$. \label{line:explore}
  \State Sample $x!t\in \calX$ from the distribution specified by $p!t$ and play it.
  \State Observe the reward $r_t$.
  \State Define reward estimator $g!t\in \dR^\calX$, where
      $g!t_x = r_t \inner{x!t}{S(p!t)^{-1}x}$
      and $S(p)\defeq \EE_{x\sim p}[xx^\trans]$. \label{line:estimator}
  \State Set the learning rate of the next round $\beta_{t+1}=\beta_t+\frac{z_t}{\beta_t h_t}$ where $h_t=\varphi_a(\phat!t)$. \label{line:learning_rate}
\EndFor
\end{algorithmic}
\end{algorithm}

While the algorithm is from prior work, we again provide a novel analysis for the game setting, using the same instance-dependent parameters $\DRmin$, $\DCmin$, $\DeltaMix$ defined in \Cref{h2:tsallis-inf} (note that in these definitions, $\calX$ and $\calY$ are now general sets instead of the standard basis sets.). %

\begin{theorem}\label{thm:tsallis-bilinear}
    With $\alpha=1-\frac{1}{4\log m_x}$, $\beta_1= \frac{8 c d_x}{1-\alpha}$, $\bar{\beta}= \frac{32d_x}{(1-\alpha)^2 \beta_1}$, and a suitable choice of $\pzero$, in any bilinear game, against any adaptive adversary, \TsallisSPM{} achieves a worst-case regret bound of $\psmr_T=\order\rbrm[\big]{\sqrt{T d_x \log m_x} + m_x \log m_x}$.
    Additionally,
    \begin{itemize}[noitemsep,topsep=0pt]
        \item If the game has a strict PSNE, we further have $\psmr_T=\order\rbrm[\big]{\frac{d_x \log m_x}{\DRmin\DCmin}\log T + m_x \log m_x}$.
        \item If the game has no PSNEs, we further have $\psmr_T = \order\rbrm[\big]{\frac{d_x \log m_x}{\DeltaMix} + m_x \log m_x}$.
    \end{itemize}
\end{theorem}

The proof is similar to \Cref{thm:tsallis} and is deferred to \Cref{h1a:proof-tsallis-bilinear}.
Importantly, all $T$-dependent terms have only logarithmic dependence on $m_x$.
The additive $m_x\log m_x$ term is not ideal but was present in~\citet{ito2024adaptive} already.
On the other hand, if one were to ignore the linear structure, treat this as a normal game with $m_x$ actions (for the learner), and apply Tsallis-INF, then the bound would have a polynomial dependency, as either $\sqrt{m_x T}$ (worst-case) or $\frac{1}{\DCmin}m_x\log T$ (instance-dependent).

\subsection{Informed Learning in Bilinear Games with \PureLinUCB{}}\label{h2:linucb}

Finally, we turn to the informed feedback setting for general bilinear games.
Analogous to our approach in \Cref{h1:pure-ucb}, we extend the classical LinUCB algorithm \citep{abbasi-yadkori2011improved} to the game-theoretic setting.

The algorithm is based on the observation that the reward of a bilinear game is linear in its matrix $A$: $u(x,y)=\inner{x}{Ay}=\inner{A}{x y^\trans}$. The feedback $r_t$ in each round is therefore a noisy linear projection of $A$. 
The algorithm can thus maintain an estimate for $A$, utilize a LinUCB-style confidence ellipsoid to determine an optimistic reward estimate for each action pair, and play the pure maximin action based on the upper confidence bounds, analogous to \PureUCB.
The formal description of this algorithm, which we name \PureLinUCB{}, is deferred to \Cref{h1a:proof-purelinucb}.

To analyze this algorithm, we let $\DeltaLin>0$ be such that $\vStar - u(x,y)\in (-\infty,0]\cup [\DeltaLin, +\infty)$. This parameter always exists since $\calX$ and $\calY$ are finite. In the special case of normal-form games, the parameter $\DeltaLin$ would be equal to $\min_{x,y} \cbr{\Delta_{xy}\mid \Delta_{xy}>0}$. 

We show that \PureLinUCB{} effectively exploits the low-dimensional nature of the bilinear game.
The full analysis is deferred to \Cref{h1a:proof-purelinucb}, with
the main result summarized below.

\begin{theorem}\label{thm:purelinucb-informal}
    For any bilinear game, against any adaptive adversary, \PureLinUCB{} achieves a worst-case regret of $\psmr_T=\order\rbrm[\big]{d_x d_y \sqrt{T} \log T}$ and an instance-dependent regret of $\psmr_T=\order\rbrm[\big]{\frac{1}{\DeltaLin}d_x^2 d_y^2 \log^2 T}$ simultaneously.
\end{theorem}

Note that the bounds scale polynomially with the feature dimensions $d_x$ and $d_y$, but have no dependence on $m_x$ and $m_y$ at all (unlike \Cref{thm:tsallis-bilinear}).
Once again, if one were to treat this as a normal-form game and apply \PureUCB{}, then its bound suffers a polynomial dependency on $m_x$ and $m_y$.

\section{Conclusion and Future Research}\label{h1:discuss}

In this work, we investigated the problem of learning to play zero-sum games against an arbitrary adversary under bandit feedback. We introduced the PSMR metric to quantify the learner's performance against the safety value a deterministic learner can achieve against the worst-case adversary. We provide comprehensive results, tackling both uninformed and informed feedback models and covering both normal-form and bilinear games.

Our results point toward several promising future directions. First,
as discussed, the regret bounds of Tsallis-INF and \PureUCB{} are incomparable. 
    Is it possible to achieve both simultaneously in the informed setting?
    Taking one step further, what are the instance-optimal bounds for this setting?

Second, while our results for general bilinear games are analogous to those for normal-form games, it has been shown that in the special case of stochastic linear bandits ($|\calY|=1$), such bounds are not instance-optimal~\citep{lattimore2017end}.
    Therefore, understanding what fundamental quantity characterizes the difficulty of a general bilinear game is a promising direction.

Finally, our results are restricted to $\psmr$. While it coincides with $\NR$ when a PSNE exists, investigating how to achieve similar results for $\NR$ in the general case is an interesting direction.

\bibliography{ref}

\newpage
\appendix
\crefalias{section}{appendix} %
\section{Notes on Uninformed and Informed Bandit Feedback}\label{h1a:uninformed-informed}

Existing literature on bandit-feedback learning in games exhibits a subtle but critical difference in the feedback model. For instance, the self-play framework in \citet{ito2025instancedependent} assumes the learner observes only the reward of the realized actions.
In contrast, the algorithms proposed by \citet{odonoghue2021matrix}, including variants of UCB, Thompson Sampling, and K-learning, operate under the assumption that the learner observes the reward and, in addition, the action taken by the opponent.
As this distinction is rarely addressed by prior work,
we introduce the terms ``uninformed'' and ``informed'' to formally differentiate between these two settings based on the availability of the opponent's action.

A similar distinction on informedness exists in the broader field of learning Markov games.
A Markov game is a game where the utility depends on the action of players and a state, and the state is changed every round in a Markovian way determined by the players' actions.
Our setting is a special case of Markov games with only one state.
Post-hoc observations on opponents' actions is assumed by some research \citep{wei2017online, bai2020provable, xie2020learning, liu2021sharp},
but in some other work, this information is hidden \citep{tian2021online, mao2022improving, cai2023uncoupled, appel2025regret}.
Out of these papers, \citet{cai2023uncoupled} define ``uncoupled'' learning, and \citet{mao2022improving} use the term ``decentralized'' learning, both of which imply the lack of observation of the opponent's actions (our ``uninformed'' feedback).

\section{Technical Lemmas}

\begin{lemma}\label{lem:sqrt-func}
    For $a,b>0$, the function $f(x)=\sqrt{ax}-bx$ has a maximum value of $\frac{a}{4b}$ when $x=\frac{a}{4b^2}$.
\end{lemma}
\begin{proof}
    The function has a critical point when the derivative $f'(x)=\frac{\sqrt{a}}{2\sqrt{x}}-b$ is zero. Solving $f'(x^*)=0$ yields $x^*=\frac{a}{4b^2}$, and plugging it back in gives $f(x^*)=\frac{a}{4b}$. We can indeed validate that $f'(x)>0$ when $x<x^*$ and $f'(x)<0$ when $x>x^*$.
\end{proof}

\begin{lemma}\label{lem:self-bound}
    If $x,a,b,c\geq0$, and $x\leq \sqrt{a x + b} + c$, then $x\leq a+\sqrt{b}+2c$.
\end{lemma}
\begin{proof}
    In $x-c\leq \sqrt{ax+b}$, if $x\leq c$, then the result holds trivially. We assume $x>c$; squaring both sides yields $(x-c)^2\leq ax+b$. Replacing the inequality sign with an equality solves into $x=\frac a2+c\pm\sqrt{\frac{a^2}{4}+b+ac}$; the larger root is the upper bound for $x$.
    
    We now only need to prove that $\sqrt{\frac{a^2}{4}+b+ac}\leq \frac a2+\sqrt{b}+c$. Since $\rbrm[\big]{\frac a2+\sqrt{b}+c}^2-\rbrm[\big]{\frac{a^2}{4}+b+ac}=a\sqrt{b}+2c\sqrt{b}+c^2\geq 0$, this inequality holds.
\end{proof}

\begin{lemma}\label{lem:kl-bernoulli}
    Let $\Twopoint(x)$ be the two-point distribution on $\{\pm 1\}$ with expectation $x$, which puts probability $\frac{1+x}{2}$ on $+1$ and $\frac{1-x}{2}$ on $-1$. When $|a|, |b|\leq \frac12$, we have $\KL\rbrm[\big]{\Twopoint(a)\,\|\,\Twopoint(b)}\leq (a-b)^2$, where $\KL$ is the KL divergence.
\end{lemma}
\begin{proof}
    Expand the KL divergence formula:
    \begin{equation*}
        \KL\rbrm[\big]{\Twopoint(a)\,\|\,\Twopoint(b)}=\frac{1+a}{2}\ln\frac{1+a}{1+b}+\frac{1-a}{2}\ln\frac{1-a}{1-b}.
    \end{equation*}
    Define
    \begin{equation*}
        f_a(b)=\frac{1+a}{2}\ln\frac{1+a}{1+b}+\frac{1-a}{2}\ln\frac{1-a}{1-b} - (a-b)^2.
    \end{equation*}
    We need to prove that $f_a(b)\leq 0$. It is clear that $f_a(a)=0$. Taking a derivative of $f_a$,
    \begin{equation*}
        f'_a(b)=-\frac{1+a}{2(1+b)}+\frac{1-a}{2(1-b)}-2(b-a)=(b-a)\rbrm[\Big]{\frac{1}{1-b^2}-2}.
    \end{equation*}

    Since $\abs{b}\leq \frac12$, we know that the second term $\frac{1}{1-b^2}-2$ is negative. Therefore, $f'_a(b)<0$ when $b>a$ and $f'_a(b)>0$ when $b<a$, and thus $f'_a$ is maximized at $b=a$ with a maximum value of 0.
\end{proof}

\begin{lemma}\label{lem:delta-comparison}
    Let $A=\rbrm[\Big]{\begin{matrix}a & b \\ c & d\end{matrix}}$ be the utility matrix for a $2\times 2$ zero-sum matrix game with all entries in $[-1, 1]$. Assume it has a unique mixed-strategy Nash equilibrium $(p^*,q^*)$ of value $\vMix$.
    Let $\DeltaEntry=\min\{|a-b|,|a-c|,|b-d|,|c-d|\}$ be the minimum entry gap. Then, $\DeltaMix \geq \DeltaEntry^2/4$.
\end{lemma}
\begin{proof}
    Since $A$ does not have pure-strategy Nash equilibriums, we have $\DeltaEntry>0$. Up to permutations of rows and columns, we assume $a>b$, $d>b$, $d>c$, $a>c$, and $b\geq c$. The pure-strategy maximin value is $\vStar=\max\{\min\{a,b\},\min\{c,d\}\}=b$.

    Using the indifference condition of a mixed NE, we can solve the NE
    \begin{align*}
        p^* &= \rbrm[\Big]{\frac{d-c}{a-b-c+d},\frac{a-b}{a-b-c+d}},\\
        q^* &= \rbrm[\Big]{\frac{d-b}{a-b-c+d},\frac{a-c}{a-b-c+d}}.
    \end{align*}
    Therefore, the Nash value is $\vMix=\anglem[\big]{p^*,a q^*}=\frac{ad-bc}{a-b-c+d}$, and the gap parameter $\DeltaMix=\vMix-\vStar=\frac{ad-bc}{a-b-c+d}-b=\frac{(a-b)(d-b)}{a-b-c+d}$.

    By the definition of $\DeltaEntry$, we know that $a-b\geq \DeltaEntry$, $d-b\geq \DeltaEntry$. Together with the fact that $a-b-c+d=(a-b)+(d-c)\leq 2+2=4$, the claim of the lemma is proved.
\end{proof}

\section{Proof of \Cref{thm:uninformed-lb}}\label{h1a:proof-lower-bound}

In this section, we provide full mathematical rigor to the informal proof in \Cref{h2:uninformed-lb}.

Consider the games determined by the following two matrices,
\begin{equation*}
    A=\begin{pmatrix}
        0 &        \DeltaC \\ 
        -\DeltaR & K - \DeltaR 
      \end{pmatrix},
    \quad 
    B=\begin{pmatrix}
        -\DeltaR & K - \DeltaR \\
        0 &        \DeltaC \\ 
      \end{pmatrix},
\end{equation*}
parameterized by $K$, $\DeltaR$, and $\DeltaC$. We assume $\DeltaR, \DeltaC \in (0, 1]$ and $K-\DeltaR\in [-1, 1]$. The two gap variables $\DeltaR,\DeltaC$ are exactly $\DRmin$ and $\DCmin$ of our matrices. We set the noise model to be binary on $\{\pm 1\}$, i.e., at round $t$, $\eps_t$ comes from a two-point distribution so that $r_t\in \{-1,+1\}$ and $\EE[r_t]=u(x!t,y!t)$; we denote this distribution as $r_t\sim \Twopoint(u(x!t,y!t))$.

We will show that any uninformed algorithm suffers $\psmr_T = \Omega\rbrm[\big]{\min\cbrm[\big]{\frac{1}{\DRmin\DCmin},\sqrt{T}}}$ for some adversary and for at least one of the games defined by $A$ and $B$.

From the definitions of $A$ and $B$, we know that $A$ admits a strict PSNE $(x^*,y^*)=(e_1,e_1)$, and $B$ admits a strict PSNE $(x^*,y^*)=(e_2,e_1)$. We have $v^*=0$ for both games. We design an adversary parameterized by $\eps\in (0,1)$ and $T'\in \bbN_+$, $T'\leq T$, where he plays a sample from the distribution $y!t\sim q!t=(1-\eps,\eps)$ for $t\leq T'$ and $y!t=y^*$ for $t>T'$. Then, for $t\leq T'$,
\begin{equation*}
    A q!t = \begin{pmatrix}
        \eps \DeltaC\\
        K \eps - \DeltaR
    \end{pmatrix},
    \quad
    B q!t = \begin{pmatrix}
        K \eps - \DeltaR\\
        \eps \DeltaC
    \end{pmatrix}.
\end{equation*}

Let $\delta=\eps\DeltaC-K \eps+\DeltaR$. In each round, the KL divergence between distributions corresponding to $A q!t$ and $B q!t$ is either $\KL(\Twopoint(\eps \DeltaC)\,\|\,\Twopoint(K\eps - \DeltaR))$ or $\KL(\Twopoint(K\eps - \DeltaR)\,\|\,\Twopoint(\eps \DeltaC))$, both bounded by $\delta^2$ by \Cref{lem:kl-bernoulli}. Using the divergence decomposition theorem (standard tool from bandit lower bounds, see e.g. Chapter~15.1 of \citet{lattimore2020bandit}), the total KL divergence between the probability measures at the end of the $T'$-th round is bounded by $\delta^2T'$. Let $\PP_A, \PP_B$ denote these two measures.

If $\delta^2 T'$ is bounded by a constant, say $\delta^2 T'\leq 1$, then the algorithm cannot distinguish these two measures beyond a constant accuracy. This is captured in the Bretagnolle-Huber inequality \citep{bretagnolle1979estimation}: for any event $S$ and its complement $S^\compl$,
\begin{equation*}
    \PP_A[S]+\PP_B[S^\compl]\geq \frac12 \exp(-\KL(\PP_A\,\|\,\PP_B)).
\end{equation*}

Assuming $\delta^2 T'\leq 1$, we take $S=[x!t=e_2]$, $S^\compl=[x!t=e_1]$ and sum this inequality over $t=1,2,\dots,T'$. We obtain
\begin{equation*}
    \EE_A\sbrm[\bigg]{\sum_{t=1}^{T'} \one[x!t=e_2]}+\EE_B\sbrm[\bigg]{\sum_{t=1}^{T'} \one[x!t=e_1]}\geq \frac12 \exp(-1) T' \geq \frac16 T'.
\end{equation*}
At least one expectation in the LHS is larger than $\frac1{12}T'$; without loss of generality, we assume it is the first one. Then $\EE_A\sbrm[\big]{\frac{1}{T'}\sum_{t=1}^{T'} p!t_2}\geq \frac1{12}$, or in English, the algorithm chooses the second arm more than $\frac1{12}$ of the time in expectation. Assume $K\eps<\DeltaR$, and we then have
\begin{align*}
    \psmr_T
    & = \EE_A\sbrm[\bigg]{\sum_{t=1}^{T}(v^*-u_A(x!t,y!t))}
    = \EE_A\sbrm[\bigg]{\sum_{t=1}^{T}\rbrm[\Big]{0-\anglem[\big]{p!t,A q!t}}} \\
    & \geq \EE_A\sbrm[\bigg]{\sum_{t=1}^{T'}-\anglem[\big]{p!t,A q!t}} \tag{$\psmr\geq 0$ when $t>T'$ and $y!t=y^*$ is a NE} \\
    & = \EE_A\sbrm[\bigg]{\sum_{t=1}^{T'}-\rbrm[\big]{p!t_1 \eps \DeltaC+p!t_2 (K\eps-\DeltaR)}}
    \geq T'\rbrm[\Big]{\frac{1}{12} (\DeltaR-K\eps)-\frac{11}{12} \eps\DeltaC}. \yestag\label{eq:psmr-lb}
\end{align*}
The other case can be proven by replacing $A$ with $B$ and swapping actions $1$ and $2$, arriving at the same bound.

For any tuple $(\DeltaR, \DeltaC, T)$ with $\DeltaR,\DeltaC<\frac 1{13}$, $T\geq 169$, we construct $(K,\eps,T')$ as follows:
\begin{itemize}
    \item If $\frac{1}{\DeltaR\DeltaC}<13\sqrt{T}$, we let $K=1$, $\eps=\DeltaR-12\DeltaR\DeltaC$, and $T'=\rbrm[\big]{\frac{1}{13\DeltaR\DeltaC}}^2$.
    We can check that our assumptions are satisfied with $K-\DeltaR>0>-1$, $K\eps=\DeltaR-12\DeltaR\DeltaC<\DeltaR$ and $\delta=13\DeltaR\DeltaC-12\DeltaR\DeltaC\DeltaC<\frac{1}{\sqrt{T'}}$.
    Our PSMR bound \eqref{eq:psmr-lb} expands to $\psmr_T\geq \frac{T'}{12}(12\DeltaR\DeltaC-11\DeltaR\DeltaC+132\DeltaR\DeltaC\DeltaC)\geq T'\DeltaR\DeltaC/12=\frac{1}{2028\DeltaR\DeltaC}$.
    \item If $\frac{1}{\DeltaR\DeltaC}\geq13\sqrt{T}$, we let $T'=T$, $\eps=\min\cbrm[\big]{1,\frac{1}{13\DeltaC\sqrt{T}}}\leq\DeltaR$, and $K=\frac{\DeltaR-\frac{12}{13\sqrt{T}}}{\eps}$. Again, our assumptions are satisfied with $K\leq \frac{\DeltaR}{\eps}\leq 1$, $K\geq\min\cbrm[\Big]{\frac{\DeltaR-\frac{12}{13\sqrt{T}}}{1},\frac{\DeltaR-\frac{12}{13\sqrt{T}}}{(13\DeltaC\sqrt{T})^{-1}}}\geq \min\cbrm[\big]{\DeltaR-\frac{12}{13\sqrt{T}},13\DeltaR\DeltaC\sqrt{T}-12\DeltaC}\geq -\frac{12}{13}$ which implies $K-\DeltaR>-1$, $K\eps=\DeltaR-\frac{12}{13\sqrt T}<\DeltaR$, and $\delta=\frac{12}{13\sqrt T}+\min\rbrm[\big]{\frac{1}{13 \sqrt{T}},\DeltaC}\leq\frac{1}{\sqrt T}$. Therefore, \eqref{eq:psmr-lb} expands to $\psmr_T\geq \frac{T}{12}\rbrm[\big]{\frac{12}{13\sqrt T}-\frac{11}{13\DeltaC\sqrt T}\DeltaC}\geq \frac1{156}\sqrt T$.
\end{itemize}
Together, we can prove the final lower bound $\psmr_T\geq \frac{1}{156}\min\cbrm[\big]{\frac{1}{13\DRmin\DCmin},\sqrt{T}}$. \qed

\section{Proof of \Cref{thm:pureucb}}\label{h1a:proof-pureucb}

In this section, we give a formal proof of our regret analysis of \PureUCB.

Let $x^*=\argmax_{x\in \calX} \min_{y\in \calY} u(x,y)$ and $y^* = \argmin_{y\in \calY} u(x^*,y)$ with ties broken arbitrarily. We know that $\vStar=u(x^*,y^*)$.

In each round, we know that $r_t$ is a noisy observation of $u(x!t,y!t)$. \Cref{code:ucb-n,code:ucb-sum,code:ucb-confidence} of \Cref{alg:pure-ucb} uses these observations to build a mean estimator for $u(x,y)$ and a confidence bound. Like the usual analysis of the UCB algorithm in stochastic MAB, we start by showing that $U!t$ is indeed an upper confidence bound on $u$, by showing the following lemma. We show its proof for completeness and for the extra details needed to handle the potential stochastic nature of the adversary.
\begin{lemma}\label{lem:ucb-ucb}
    Fix $(x,y)\in \calX\times \calY$. With probability at least $1-\delta$, the confidence bound $\absm[\Big]{u(x,y) - \frac{R!{t}(x,y)}{N!{t}(x,y)}}\leq \sqrt{\frac{4\log (1/\delta)+2\log (1+N!{t}(x,y))}{N!{t}(x,y)}}$ holds simultaneously for all $t$ such that $N!t(x,y)>0$.
\end{lemma}
\begin{proof}
    Denote as a shorthand $N_t=N!t(x,y)$ and $I_t=N_t-N_{t-1}$. Let $\calF_{t-1}$ be the $\sigma$-algebra generated by everything before $r_t$ is drawn, i.e., $(x!1,y!1,\dots,x!t,y!t,r_1,\dots,r_{t-1})$. By the UCB algorithm, $I_t$ and thus $N_t$ is measurable on $\calF_{t-1}$.
    
    By the definition of the noise model, the stochastic process $M=\{R!t(x,y)-N!t(x,y) u(x,y)\}_t$ is a martingale. Its difference sequence $\{M_t-M_{t-1}\}$ is conditionally 1-subgaussian when $(x!t,y!t)=(x,y)$, and is equal to 0 otherwise. Let $Z_t(\lambda)=\exp(\lambda M_t - \frac12 \lambda^2 N_t)$. Conditional on $\calF_{t-1}$, if $I_t=0$, then $Z_t(\lambda)=Z_{t-1}(\lambda)$; otherwise $Z_t(\lambda)=Z_{t-1}(\lambda)\exp(\lambda \eps_t - \frac12 \lambda^2)$. Combining both cases leads to $\EE\sbr{Z_t(\lambda)\mid\calF_{t-1}}\leq Z_{t-1}(\lambda)$, and thus $\{Z_t(\lambda)\}_t$ is a supermartingale.

    By the method of mixtures, $\Zbar_t=\int_\dR Z_t(\lambda) d\Phi(\lambda)$ is also a supermartingale, where $\Phi$ is the CDF of $\calN(0, 1)$. This integration evaluates to
    \begin{equation*}
        \Zbar_t=\int_\dR \exp(\lambda M_t - \frac12 \lambda^2 N_t) \frac{1}{\sqrt{2\pi}}\exp(-\frac 12 \lambda^2) d\lambda = \frac{1}{\sqrt{1+N_t}} \exp\rbrm[\Big]{\frac{M_t^2}{2(1+N_t)}}.
    \end{equation*}

    Since $Z_0(\lambda)=1$, we have $\Zbar_0=1$, and we can use Ville's inequality to get
    \begin{equation*}
        \PP\sbrm[\big]{\exists t: \Zbar_t\geq \frac{1}{\delta}}\leq \delta.
    \end{equation*}
    Plugging in, taking logarithms on both sides, and taking the complement on the event yields that with probability at least $1-\delta$, for all $t$,
    \begin{equation*}
        -\frac12 \log(1+N_t) + \frac{M_t^2}{2(1+N_t)} \leq \log(1/\delta),
    \end{equation*}
    which, when $N_t>0$, simplifies to
    \begin{equation*}
        \frac{|M_t|}{N_t}\leq \sqrt{\frac{2(1+N_t)}{N_t^2}\rbrm[\big]{\log(1/\delta)+\frac12 \log(1+N_t)}}.
    \end{equation*}
    Loosely bounding $\frac{1+N_t}{N_t}$ by 2 and plugging in our shorthands, we see that
    \begin{equation*}
        \frac{\absm[\big]{R!t(x,y)-N!t(x,y) u(x,y)}}{N!t(x,y)} \leq \sqrt{\frac{2}{N!t(x,y)}\rbrm[\big]{2\log(1/\delta)+\log(1+N!t(x,y))}},
    \end{equation*}
    which is equivalent to the desired result, up to algebraic simplification.
\end{proof}

Back to \Cref{thm:pureucb}. We first assume the high-probability event of \Cref{lem:ucb-ucb} happens for all action pairs, which implies $u(x,y)\leq U!t(x,y)$ for all $x$, $y$, and $t$. Using the confidence bound on the pair $(x^*,y)$, we see that
\begin{align*}
    \vStar & = u(x^*,y^*) 
    = \min_{y\in \calY} u(x^*,y)
    \leq \min_{y\in \calY} U!{t}(x^*,y) \\
    & \leq \min_{y\in \calY} U!{t}(x!t,y) \tag{$x!t$ maximizes $x\mapsto \min_{y\in \calY}U!{t}(x,y)$}  \\
    & \leq U!{t}(x!t,y!t). \yestag\label{eq:played-ucb-geq-vstar}
\end{align*}

Consider each action pair $(x,y)$ such that $\Delta_{xy}>0$. At the last iteration this action pair is played, i.e., at the maximum $t$ such that $(x!t,y!t)=(x,y)$, we know that $U!{t}(x!t,y!t)\geq \vStar=u(x,y)+\Delta_{xy}$, so
\begin{equation*}
    \sqrt{\frac{4\log (1/\delta)+2\log (1+N!{t-1}(x,y))}{N!{t-1}(x,y)}}\geq \Delta_{xy}.
\end{equation*}
which simplifies to $N!{t-1}(x,y)\leq \frac{1}{\Delta_{xy}^2} \rbrm[\big]{4\log (1/\delta)+2\log (1+N!{t-1}(x,y))}$, and can be loosely upper-bounded as $\frac{6}{\Delta_{xy}^2} \log T$. Thus, $N!t(x,y)$ is at most this value plus one. Since this is the last round $(x,y)$ is played, the same bound applies to $N!T(x,y)$.

Therefore, our regret is bounded as
\begin{align*}
    \psmr_T 
    & = \EE\sbrm[\Big]{\sum_{t=1}^T \rbrm[\big]{\vStar-u(x!t,y!t)}}
    = \EE\sbrm[\Big]{\sum_{(x,y)\in \calX\times \calY} N!T(x,y) \rbrm[\big]{\vStar-u(x,y)}} \\
    & = \EE\sbrm[\Big]{\sum_{(x,y)\in \calX\times \calY} N!T(x,y) \Delta_{xy}}
    \leq \EE\sbrm[\Big]{\sum_{(x,y) : \Delta_{xy}>0} N!T(x,y) \Delta_{xy}} \yestag\label{eq:ucb-wc-step}\\
    & \leq \sum_{(x,y) : \Delta_{xy}>0} \rbrm[\Big]{1+\frac{6}{\Delta_{xy}^2} \log T} \Delta_{xy}.
\end{align*}
This directly implies our instance-dependent bound.

To obtain our worst-case bound, we introduce a constant $\Delta>0$. For action pairs with $0<\Delta_{xy}<\Delta$, starting from \eqref{eq:ucb-wc-step}, we bound the total number of plays as $\sum_{(x,y):0<\Delta_{xy}<\Delta} N!T(x,y)\leq T$; for each action pair with $\Delta_{xy}>\Delta$, the regret contributed by such a pair is bounded by $\rbrm[\big]{1+\frac{6}{\Delta_{xy}^2} \log T} \Delta_{xy}\leq 2+\frac{6 \log T}{\Delta_{xy}}\leq 2+ \frac{6 \log T}{\Delta}$. Together, we have
\begin{align*}
    \psmr_T
    & \leq \Delta T + \sum_{(x,y) : \Delta_{xy}>\Delta} \rbrm[\Big]{2+ \frac{6 \log T}{\Delta}} \\
    & \leq \Delta T + m_x m_y \rbrm[\Big]{2+ \frac{6 \log T}{\Delta}} \\
    & = 2 m_x m_y + 2\sqrt{6 m_x m_y \log T},
\end{align*}
where in the last step we substitute in with $\Delta=\sqrt{\frac{6 m_x m_y \log T}{T}}$

Note that both results are conditional on the high-probability event in \Cref{lem:ucb-ucb} holding for all action pairs $(x,y)$, which happens with probability at least $1-m_x m_y \delta$. If the confidence event does not happen, we trivially bound the PSMR by $T$, and when $\delta=\frac{1}{T}$, this low-probability event adds an extra term of $m_x m_y$ to both bounds of PSMR; in both bounds, the extra term is absorbed by the big-$\order$ notation. \qed

\section{Proof of \Cref{thm:tsallis-bilinear}}\label{h1a:proof-tsallis-bilinear}

In this section, we provide a full proof of \Cref{thm:tsallis-bilinear}.

We first specify the exploration distribution $\pzero$ used in the algorithm. To recall, the variance function $S(\cdot)$ is defined as
\begin{equation*}
    S(p)\defeq \EE_{x\sim p}[xx^\trans].
\end{equation*}
Given a set $\calX$, there exists a distribution $p\in \simplex^\calX$ such that $\inner{x}{S(p)^{-1}x}\leq d_x$ for every $x\in \calX$; this is a result known as the Kiefer-Wolfowitz theorem \citep{kiefer1960equivalence}. We choose a $\pzero\in \simplex^\calX$ such that
\begin{equation}
    \forall x\in \calX, \inner{x}{S(\pzero)^{-1}x}\leq c d_x
\end{equation}
for some constant $c=\order(1)$.

As in the proof of \Cref{thm:tsallis}, we use $C=\EE\sbrm[\big]{\sum_{t=1}^T \one[y!t \neq y^*]}$ to denote the expected number of times the adversary deviates from playing $y^*$.
Then, the external regret compared against~$x^*$ is lower bounded as
\begin{align}
    \ER_T(x^*)
    &=
    \EE\left[ \sum_{t=1}^T \left( u(x^*,y!t) - u(x!t,y!t) \right) \right]
    \nonumber \\ 
    &=
    \EE\Bigg[ 
        \sum_{t=1}^T \left( u(x^*,y^*) - u(x!t,y^*) \right) 
        +
        \sum_{t=1}^T \left( u(x^*,y!t) - u(x^*,y^*) \right) 
        \nonumber \\ 
        &\qquad\qquad\qquad+
        \sum_{t=1}^T \left( u(x!t,y^*) - u(x!t,y!t) \right) 
    \Bigg]
    \nonumber \\
    &
    =
    \EE\left[ 
        \sum_{t=1}^T \langle p_x!t, \DeltaR_{x} \rangle
        +
        \sum_{t=1}^T \langle p_y!t, \DeltaC_{y} \rangle
        -
        \sum_{t=1}^T \left( u(x!t,y^*) - u(x!t,y!t) \right) 
    \right]
    \nonumber \\ 
    &
    \geq
    \EE\left[ 
        \sum_{t=1}^T \langle p_x!t, \DeltaR_{x} \rangle
        -
        2 C
    \right]
    ,
    \label{eq:ereg_lower_bilinear}
\end{align}
where the last inequality follows from the definition of $C$ and the assumption that $\|x\|_2 \leq 1$ and $\|y\|_2 \leq 1$ for any $x \in \mathcal{X}$ and $y \in \mathcal{Y}$.

We then upper bound the external regret by using the upper bound of \TsallisSPM{} \citep{ito2024adaptive}\footnote{The arXiv and published versions of \citet{ito2024adaptive} have different numberings for their theorems and equations. We use the published version to avoid confusion.} instead of using the Tsallis-INF upper bound.
From  Eq.~(26) and the choice of parameters in Section 4.3 (see discussion around Eqs.~(35) and (36)) in \citet{ito2024adaptive}, 
the external regret is upper bounded as
\begin{align}
    &
    \ER_T(x^*)
    \nonumber \\ 
    &\leq
    O\rbr{
    \EE\sbr{
        \min\left\{
        \sqrt{\frac{d_x m_x^{1-\alpha} T}{\alpha(1-\alpha)}}
        ,\,
        \sqrt{ 
            \frac{d_x}{\alpha(1-\alpha)} \frac{m_x^{1-\alpha}}{\DRmin} 
            \sum_{t=1}^T \langle p_x!t, \DeltaR_{x} \rangle
            \log T
        }
        \right\}
    }
    \!+\! m_x \log m_x
    }
    .
    \label{eq:regret_alpha_tsallis}
\end{align}

The first $\min$ in \eqref{eq:regret_alpha_tsallis} with $\alpha=1-\frac{1}{4\log m_x}$ yields
$\ER_T(x^*) \leq O(\sqrt{T d_x \log m_x} + m_x \log m_x)$.
Hence, we obtain
$\psmr_T \leq \ER_T(x^*) \leq O(\sqrt{T d_x \log m_x} + m_x \log m_x)$.
If the game has no PSNEs, similar to \Cref{thm:tsallis}, the $\DeltaMix$-dependent bound is a direct implication of this worst-case bound, the fact that $\psmr_T\leq \ER_T-\DeltaMix T$, and \Cref{lem:sqrt-func}.

We next focus on the games with a strict PSNE.
From the second branch in the minimum in \eqref{eq:regret_alpha_tsallis} combined with Jensen's inequality and \eqref{eq:ereg_lower_bilinear}, we also have
\begin{equation*}
    \ER_T(x^*)
    \leq
    O\rbr{
    \sqrt{ 
        \frac{d_x}{\alpha(1-\alpha)} \frac{m_x^{1-\alpha}}{\DRmin} 
        \left( \ER_T(x^*) + 2 C \right)
        \log T
    }
    + m_x \log m_x
    }
    .
\end{equation*}
Plugging $\alpha=1-\frac{1}{4\log m_x}$ in the last inequality and applying \Cref{lem:self-bound},
we obtain
\begin{equation}
    \ER_T(x^*)
    \leq
    O\left(
    \frac{d_x \log m_x \log T}{\DRmin}
    +
    \sqrt{\frac{C d_x \log m_x \log T}{\DRmin}}
    +
    m_x \log m_x
    \right)
    .
    \nonumber
\end{equation}
Since we have $\psmr_T \leq \ER_T(x^*) - \DCmin C$ (from the same analysis as in \eqref{eq:er-nr-difference}), combining this with the last inequality yields
\begin{align*}
    \psmr_T 
    &
    \leq
    \ER_T(x^*) - \DCmin C
    \nonumber \\
    &\leq
    O\left(
    \frac{d_x \log m_x \log T}{\DRmin}
    +
    \sqrt{\frac{C d_x \log m_x \log T}{\DRmin}}
    +
    m_x \log m_x
    \right)
    - 
    \DCmin C
    \\
    &=
    O\left(
    \frac{d_x \log m_x \log T}{\DRmin}
    +
    \frac{d_x \log m_x \log T}{\DRmin \DCmin}
    +
    m_x \log m_x
    \right)
    ,
\end{align*}
where in the last inequality we considered the worst-case with respect to $C$ (recall \Cref{lem:sqrt-func}).
This completes the proof of \Cref{thm:tsallis-bilinear}.

\section{Informed Learning in Bilinear Games}\label{h1a:proof-purelinucb}

In this section, we detail our technical results in informed feedback learning in bilinear games, briefly mentioned in \Cref{h2:linucb}. We extend \PureUCB{} and propose \PureLinUCB{}, extending the classical LinUCB algorithm~\citep{abbasi-yadkori2011improved}; see \Cref{alg:pure-lin-ucb}.

\begin{algorithm}[h]
\caption{\PureLinUCB}\label{alg:pure-lin-ucb}
\begin{algorithmic}[1]
\Require A regularization coefficient $\lambda\in \dR_+$. A sequence of confidence radius multipliers $\beta_1, \beta_2, \dots\in \dR^+$.
\State Initialize $V!0=\lambda I_{d_x d_y}\in \dR^{(d_x d_y)\times (d_x d_y)}$, $b!0=\mathbf{0}\in \dR^{d_x d_y}$. \label{code:ridge-init}
\For{$t=1,2,\dots$}
  \State Compute $\Ahat!t=\mat_{d_x\times d_y}(\rbrm{V!{t-1}}^{-1} b!{t-1})$. \label{line:ridge}
  \State Define the UCB for any pair $(x,y)$ as\\
  \centerline{$U!t(x,y)=\anglem[\big]{x,\Ahat!t y}+\beta_t \norm{\vect (x y^\trans)}_{\rbrm{V!{t-1}}^{-1}}.$}\label{code:U-def}
  \State Let  $x!t = \argmax_{x\in \calX} \min_{y\in \calY} U!t(x,y)$ with ties broken arbitrarily.
  \label{line:pure_maximin}
  \State Play $x!t$. Observe the adversarial action $y!t$ and the reward $r_t$.
  \State Let $a!t = \vect(x!t \rbrm{y!t}^\trans)$.
  \State Update the covariance matrix and bias vector for ridge regression: $V!t=V!{t-1}+a!t \rbrm{a!t}^\trans$, $b!t=b!{t-1}+r_t a!t$. \label{code:ridge-update}
\EndFor
\end{algorithmic}
\end{algorithm}

In the definition of this algorithm, we distinguish between $\dR^{d_x\times d_y}$, the set of $d_x\times d_y$ matrices, and $\dR^{d_x d_y}$, vectors of length $d_x d_y$. Let $\vect(\cdot)$ be the ``flatten'' map that turns matrices to vectors of length $d_x d_y$, and let $\mat_{d_x\times d_y}(\cdot)$ be its inverse, turning vectors into $d_x\times d_y$ matrices.

The algorithm is based on the observation that the reward of a bilinear game is linear in its matrix $A$: $u(x,y)=\inner{x}{Ay}=\inner{\vect(A)}{\vect(x y^\trans)}$. The feedback $r_t$ in each round is thus a noisy linear projection of $\vect(A)$. 
The algorithm thus maintains a ridge regression estimate for $\vect(A)$ (\Cref{line:ridge} and \Cref{code:ridge-update}) and utilizes a LinUCB-style confidence bound to determine an optimistic reward estimate for each action pair (\Cref{code:U-def}). 
Finally, the algorithm again plays the pure maximin action based on the upper confidence bounds (\Cref{line:pure_maximin}), analogous to \PureUCB.

Recall that we define the instance-dependent parameter $\DeltaLin>0$ be such that $\vStar - u(x,y)\in (-\infty,0]\cup [\DeltaLin, +\infty)$.
We present the formal version of \Cref{thm:purelinucb-informal}:

\begin{theorem}\label{thm:purelinucb}
    \PureLinUCB{} achieves a worst-case regret of $\psmr_T=\order\rbrm[\big]{d_x d_y \sqrt{T} \log T}$ and an instance-dependent regret of $\psmr_T=\order\rbrm[\big]{\frac{1}{\DeltaLin}d_x^2 d_y^2 \log^2 T}$ simultaneously, against any adaptive adversary, by letting $\beta_t=\sqrt{\lambda d_x d_y}+\sqrt{2\log \delta^{-1}+d_x d_y \log \rbrm[\big]{1+\frac{t}{d_x d_y\lambda}}}$, $\delta=\frac{1}{T}$, and $\lambda=1$.
\end{theorem}

\begin{proof}
Let $x^*=\argmax_{x\in \calX} \min_{y\in \calY} u(x,y)$ and $y^* = \argmin_{y\in \calY} u(x^*,y)$ with ties broken arbitrarily. Define $a^*=\vect(x^* (y^*)^\trans)$. We also know that $\vStar=u(x^*,y^*)=\inner{\vect(A)}{a^*}$.

In each round, we know that $r_t$ is a noisy linear observation on $\vect(A)$, or formally, $r_t=\inner{\vect(A)}{a!t}+\eps_t$. \Cref{code:ridge-init,code:ridge-update} of \Cref{alg:pure-lin-ucb} uses these observations to build a ridge regression on $\vect(A)$, where $\Ahat!t$ is the regularized least-square estimator for $A$. The usual analysis of LinUCB, e.g. Theorem 2 of \cite{abbasi-yadkori2011improved}, is applicable here and states that $A$ is within a confidence ellipsoid centered at $\Ahat!t$ with high probability, captured by the following lemma.

\begin{lemma}[Confidence ellipsoid]\label{lem:conf-ellipsoid}
    For any fixed parameters $\lambda>0$ and $\delta\in (0, 1)$, we have for probability at least $1-\delta$, $A\in \confset_t$ for all $t=1,2,\dots$, where
    \begin{equation}
        \confset_t=\cbr{A'\in \dR^{d_x\times d_y}\mid \normm[\big]{\vect(A')-\vect(\Ahat!t)}_{V!t}\leq \beta_t},
    \end{equation}
    in which $V!t=\lambda I + \sum_{s=1}^t a!s \rbrm[\big]{a!s}^\trans$ is symmetric and positive definite, $b!t=\sum_{s=1}^t r_s a!s$, $\Ahat!t=\mat\rbrm[\big]{(V!t)^{-1}b!t}$, and $\beta_t=\sqrt{\lambda d_x d_y} +\sqrt{2\log\rbrm[\big]{\frac{1}{\delta}}+d_x d_y \log \rbrm[\big]{1+\frac{t}{d_x d_y\lambda}}}$.
\end{lemma}

Note that the sequence $\{\beta_t\}$ is monotonically increasing. Our form of the lemma differs slightly from the standard form in \cite{lattimore2020bandit}, as our $\vect(A)$ has a dimension of $d_x d_y$ and its norm is upper-bounded as $\norm{\vect(A)}_2=\norm{A}_F\leq \sqrt{d_x d_y}$, but it is a direct implication of the standard result.

We first assume that the high-probability event in \Cref{lem:conf-ellipsoid} holds. When $A\in \confset_t$, $u(x,y)$ for an arbitrary action pair is upper bounded by
\begin{align*}
    u(x,y) &\leq \max_{A'\in \confset_t} \inner{x}{A'y} \\
    & = \max_{\norm{A'-\Ahat!t}_{V!t}\leq \beta_t} \inner{\vect(x y^\trans)}{\vect(A')} \\
    & = \max_{\norm{\theta'}_2\leq \beta_t} \inner{\vect(x y^\trans)}{\vect(\Ahat!t)+(V!t)^{\sfrac {-1}2}\theta'} \\
    & = \anglem[\big]{x,\Ahat!t y} + \max_{\norm{\theta'}_2\leq \beta_t} \inner{(V!t)^{\sfrac {-1}2}\vect(x y^\trans)}{\theta'} \yestag\label{eq:conf-symmetry}\\
    & = \anglem[\big]{x,\Ahat!t y} + \beta_t \norm{\vect(x y^\trans)}_{(V!t)^{-1}} \yestag\label{eq:max-inner-norm} \\
    & = U!{t+1}(x,y).
\end{align*}
In the derivation above, $\theta'$ is defined to be $\sqrt{V!t}\rbrm[\big]{\vect(A')-\vect(\Ahat!t)}$, \eqref{eq:conf-symmetry} is due to $V!t$ being symmetric, \eqref{eq:max-inner-norm} uses $\max_{\norm{a}_2\leq C} \inner{a}{b}=C \norm{b}_2$, and the last step is the definition of $U!{t+1}$ in \Cref{code:U-def} of the algorithm.

With a similar argument, we can also prove that $u(x,y)\geq \anglem[\big]{x,\Ahat!t y} - \beta_t \norm{\vect(x y^\trans)}_{(V!t)^{-1}}$.
This shows that $\beta_t \norm{\vect(x y^\trans)}_{(V!t)^{-1}}$ is a confidence radius of $u(x,y)$.
Together with the definition of $U!{t+1}$, we can imply that $u(x,y)\geq U!{t+1}(x,y)-2\beta_t \norm{\vect(x y^\trans)}_{(V!t)^{-1}}$.

Using the same logic as \eqref{eq:played-ucb-geq-vstar} in our proof of \Cref{thm:pureucb}, we know that $\vStar\leq U!t(x!t,y!t)$ for every round $t$.
The per-round regret can be bounded as
\begin{align*}
    \vStar-u(x!t,y!t)
    & \leq U!t(x!t,y!t)-u(x!t,y!t) \\
    & \leq 2\beta_t \normm[\big]{a!t}_{(V!{t-1})^{-1}} \leq 2 \beta_T \normm[\big]{a!t}_{(V!{t-1})^{-1}}. \yestag\label{eq:linucb-per-round}
\end{align*}
The last step is due to $\{\beta_t\}$ being monotonic. Therefore, to prove a bound on the total PSMR, we only need to upper-bound the sum of the RHS.

The rest of the proof is purely algebraic and mostly identical to that of LinUCB.

We focus on the worst-case regret first. First, use Cauchy-Schwarz to obtain
\[
    \sum_{t=1}^T \normm[\big]{a!t}_{(V!{t-1})^{-1}}\leq \sqrt{T \sum_{t=1}^T \normm[\big]{a!t}^2_{(V!{t-1})^{-1}}}.
\]
The summand in the RHS is similar to $\normm[\big]{a!t}^2_{(V!{t})^{-1}}$, which can be bounded as
\[
    \normm[\big]{a!t}^2_{(V!{t})^{-1}} 
    = \anglem[\big]{(V!{t})^{-1},a!t (a!t)^\trans} 
    = \anglem[\big]{(V!{t})^{-1},V!t - V!{t-1}}
    \leq \ln \det V!t - \ln \det V!{t-1},
\]
where the last step uses the concavity of the matrix function $M\mapsto \ln\det(M)$, the gradient of which is exactly $M^{-1}$.

By a telescoping argument, $\sum_{t=1}^T \normm[\big]{a!t}^2_{(V!{t})^{-1}}\leq \ln\det V!T-\ln\det V!0$. The initial value is known as $\ln\det V!0=d_x d_y \ln\lambda$. The final value can be bounded with the AM-GM inequality over its eigenvalues,
\begin{align*}
    \det V!T
    & \leq \rbrm[\bigg]{\frac{\trace(V!T)}{d_x d_y}}^{d_x d_y}
    = \rbrm[\bigg]{\frac{\lambda d_x d_y + \sum_{t=1}^{T} \trace\rbrm[\big]{a!t (a!t)^\trans}}{d_x d_y}}^{d_x d_y} \\
    & = \rbrm[\bigg]{\lambda + \frac{\sum_{t=1}^{T} (a!t)^\trans a!t}{d_x d_y}}^{d_x d_y}
    \leq \rbrm[\bigg]{\lambda + \frac{T}{d_x d_y}}^{d_x d_y},
\end{align*}
where the last step is due to $\normm[\big]{a!t}_2=\normm[\big]{x!t (y!t)^\trans}_F=\normm[\big]{x!t}_2\normm[\big]{y!t}_2\leq 1$. Subtracting the initial and final $\ln\det$ values, we have $\sum_{t=1}^T \normm[\big]{a!t}^2_{(V!{t})^{-1}}\leq d_x d_y \ln \rbrm[\big]{1+\frac{T}{\lambda d_x d_y}}$.

The formula we hope to bound, $\sum_{t=1}^T \normm[\big]{a!t}^2_{(V!{t-1})^{-1}}$, is very close to this value:
\begin{align*}
    \sum_{t=1}^T \normm[\big]{a!t}^2_{(V!{t-1})^{-1}}-\sum_{t=1}^T \normm[\big]{a!t}^2_{(V!{t})^{-1}}
    & = \sum_{t=1}^T \anglem[\big]{(V!{t-1})^{-1}-(V!t)^{-1},a!t (a!t)^\trans} \\
    & \leq \sum_{t=1}^T \anglem[\big]{(V!{t-1})^{-1}-(V!t)^{-1},I} \\
    & = \anglem[\big]{(V!{0})^{-1}-(V!T)^{-1},I} \\
    & \leq \anglem[\big]{(V!{0})^{-1},I}=\frac{d_x d_y}{\lambda},
\end{align*}
where we used the facts that matrices $(V!{t-1})^{-1}-(V!t)^{-1}$ and $V!T$ are positive semidefinite, and that $\normm[\big]{a!t}_2\leq 1$.

Combining all these results proves
\begin{align}
    \psmr_T
    \leq 2\beta_T \sqrt{T \rbrm[\Big]{d_x d_y \ln \rbrm[\big]{1+\frac{T}{\lambda d_x d_y}}+\frac{d_x d_y}{\lambda}}}.\label{eq:linucb-psmr-worst-case}
\end{align}

For the instance-dependent regret, notice that the gap parameter guarantees that $\vStar-u(x!t,y!t)\not\in (0,\DeltaLin)$, thus
\begin{align*}
    \vStar-u(x!t,y!t)
    \leq \frac{1}{\DeltaLin}\rbrm[\big]{\vStar-u(x!t,y!t)}^2
    \leq \frac{4}{\DeltaLin} \beta_T^2 \normm[\big]{a!t}_{(V!{t-1})^{-1}}^2.
\end{align*}
We have already proved that
\begin{align*}
    \sum_{t=1}^T \normm[\big]{a!t}_{(V!{t-1})^{-1}}^2
    \leq
    d_x d_y \ln \rbrm[\big]{1+\frac{T}{\lambda d_x d_y}}+\frac{d_x d_y}{\lambda}.
\end{align*}
Combining it with previous results, and we can prove that
\begin{equation}
    \psmr_T
    \leq \frac{4}{\DeltaLin}\beta_T^2 \rbrm[\Big]{d_x d_y \ln \rbrm[\big]{1+\frac{T}{\lambda d_x d_y}}+\frac{d_x d_y}{\lambda}}.\label{eq:linucb-psmr-instance-dep}
\end{equation}

However, \eqref{eq:linucb-psmr-worst-case} and \eqref{eq:linucb-psmr-instance-dep} hold only when the high-probability event of \Cref{lem:conf-ellipsoid} happens. By setting $\delta=\frac{1}{T}$ and trivially bounding the regret by $T$ when the high-probability event does not hold, the expected PSMR increases by at most 1. By letting $\lambda=1$, assuming $T\geq 3$ and $d_x d_y\geq 2$, and simplifying with $\log \rbrm[\big]{1+\frac{T}{d_x d_y\lambda}}\leq \log T$, we arrive at our final bounds,
\begin{align*}
    \psmr_T & \leq 17 d_x d_y \sqrt{T} \ln{T}, \\
    \psmr_T & \leq \frac{49}{\DeltaLin}d_x^2 d_y^2 \ln^2 T. \mqed
\end{align*}
\let\jmlrQED\relax
\end{proof}

\end{document}